\documentclass[12pt,letterpaper]{article}
\usepackage{color}
\usepackage{graphicx}
\usepackage{epsfig,subfigure}
\usepackage{amssymb}
\usepackage{amsthm}
\usepackage{amsmath}
\usepackage[margin=1in]{geometry}
\usepackage{float}

\usepackage{algorithm}
\usepackage{algorithmic}

\usepackage{hyperref}

\usepackage{bbm,epsfig}
\usepackage{dsfont}
\usepackage{verbatim}
\usepackage{url}

\usepackage{algorithm, algorithmic, amsmath, amssymb, graphicx}
\usepackage{epstopdf}
\usepackage{enumerate}

\newtheorem{theorem}{Theorem}
\newtheorem{proposition}{Proposition}
\newtheorem{lemma}{Lemma}

\newcommand{\iprod}[2]{\langle #1, #2 \rangle}   

\newcommand{\reals}{\mathbb{R}}

\newcommand{\vecfont}[1]{\mathbf{#1}}
\newcommand{\mat}[1]{\mathbf{#1}}

\newcommand{\note}[1]{\marginpar{\tiny *note in TeX*}}
\newcommand{\ignore}[1]{}

\renewcommand{\phi}{\varphi}

\newcommand{\twonorm}[1]{\left\| {#1} \right\|_2}

\DeclareMathOperator*{\argmin}{argmin}

\newcommand{\expec}[1]{\mathbb{E}\left[#1\right]}
\newcommand{\prob}[1]{\mathbb{P}\left[#1\right]}

\newcommand{\order}[1]{O\left({#1}\right)}

\begin{document}

\title{Alternating Minimization for Mixed Linear Regression}
\date{}
\author{
Xinyang Yi\\
{The University of Texas at Austin}\\
{yixy@utexas.edu}
\and
Constantine Caramanis\\
{The University of Texas at Austin}\\
{constantine@utexas.edu}
\and
Sujay Sanghavi \\
{The University of Texas at Austin}\\
{sanghavi@mail.utexas.edu}
%Rachel Ward \\
%{The University of Texas at Austin}\\
%{rward@math.utexas.edu}
}
% There's nothing stopping you putting the seventh, eighth, etc.
% author on the opening page (as the 'third row') but we ask,
% for aesthetic reasons that you place these 'additional authors'
% in the \additional authors block, viz.
% Just remember to make sure that the TOTAL number of authors
% is the number that will appear on the first page PLUS the
% number that will appear in the \additionalauthors section.
%\keywords{matrix completion, coherent, nuclear norm, weighted nuclear norm }

\maketitle
%\onecolumn
\begin{abstract}
Mixed linear regression involves the recovery of two (or more) unknown vectors from unlabeled linear measurements; that is, where each sample comes from exactly one of the vectors, but we do not know which one. It is a classic problem, and the natural and empirically most popular approach to its solution has been the EM algorithm. As in other settings, this is prone to bad local minima; however, each iteration is very fast (alternating between guessing labels, and solving with those labels).

In this paper we provide a new initialization procedure for EM, based on finding the leading two eigenvectors of an appropriate matrix. We then show that with this, a re-sampled version of the EM algorithm provably converges to the correct vectors, under natural assumptions on the sampling distribution, and with nearly optimal (unimprovable) sample complexity. This provides not only the first characterization of EM's performance, but also much lower sample complexity as compared to both standard (randomly initialized) EM, and other methods for this problem.
\end{abstract}

\section{Introduction}
In this paper we consider the {\em mixed linear regression} problem: we would like to recover vectors from linear observations of each, except that these are {\em unlabeled}. In particular, consider for $i=1,\ldots,N$
\[
y_i = \iprod{\vecfont{x}_i}{\vecfont{\beta}^*_1}\, z_i \,  + \, \iprod{\vecfont{x}_i}{\vecfont{\beta}^*_2}\,(1-z_i) \, + \,  w_i,
\]
where each $z_i$ is either 1 or 0, and $w_i$ is noise independent of everything else. A value $z_i = 1$ means the $i^{th}$ measurement comes from $\vecfont{\beta}^*_1$, and $z_i = 0$ means it comes from $\vecfont{\beta}^*_2$. Our objective is to infer  $\vecfont{\beta}^*_1,\vecfont{\beta}^*_2 \in \mathbb{R}^k$ given $(y_i,\vecfont{x}_i), i = 1,\ldots,N$; in particular, we do not have access to the labels $z_i$. For now\footnote{As we discuss in more detail below, some work has been done in the sparse version of the problem, though the work we are aware of does not give an efficient algorithm with performance guarantees on $\|\hat{\beta}_i - \beta^{\ast}\|$, $i=1,2$.}, we do not make {\it a priori} assumptions on the $\beta$'s; thus we are necessarily in the regime where the number of samples, $N$, exceeds the dimensionality, $k$ ($N>k$).

We show in Section \ref{sec:mainresults} that this problem is NP-hard in the absence of any further assumptions. We therefore focus on the case where the measurement vectors $\mathbf{x}_i$ are independent, uniform Gaussian vectors in $\mathbb{R}^p$. While our algorithm works in the noisy case, our performance guarantees currently apply only to the setting of no noise, i.e., $w_i = 0$. 

Mixed linear regression naturally arises in any application where measurements are from multiple latent classes and we are interested in parameter estimation. See \cite{deb2000estimates} for application of mixed linear regression in health care and work in \cite{grun2007applications} for some related dataset.

The natural, and empirically most popular, approach to solving this problem (as with other problems with missing information) is the Expectation-Maximization, or EM, algorithm; see e.g.\cite{Viele2002}. In our context, EM involves iteratively alternating between updating estimates for $\beta_1,\beta_2$, and  estimates for the labels; typically, unless there is specific side-information, the initialization is random. Each step can be solved in closed form, and hence is very computationally efficient. However, as widely acknowledged, there has been to date no way to analytically pre-determine the performance of EM; as in other contexts, it is prone to getting trapped in local minima \cite{wu1983convergence}. 

{\bf Contribution of our paper:} We provide the first analytical guarantees on the performance of the EM algorithm for mixed linear regression. A key contribution of our work, both algorithmically and for analysis, is the initialization step. In particular, we develop an initialization scheme, and show that with this EM will converge at least exponentially fast to the correct $\beta$'s and finally recover ground truth exactly, with $O(k\log^2 k)$ samples for a problem of dimension $k$. This sample complexity is optimal, up to logarithmic factors, in the dimension and in the error parameter. We are investigating the proposed algorithm in the noisy case, while in this paper we only present noiseless result.
%This holds under two natural simplifying assumptions: the measurement vectors $\vecfont{x}_i$ are independent uniform Gaussians in $\mathbb{R}^p$, and there is no noise, i.e. $w_i = 0$.

%The rest of this paper is organized as follows {\bf @@}

\subsection{Related Work}

There is of course a huge amount of work in both latent variable modeling, and finite mixture models; here we do not attempt to cover this broad spectrum, but instead focus on the most relevant work, pertaining directly to mixed linear regression.

The work in \cite{Viele2002} describes the application of the EM algorithm to the mixed linear regression problem, both with bayesian priors on the frequencies for each mixture, and in the non-parametric setting (i.e. where one does not {\it a priori} know the relative fractions from each $\beta$). More recently, in the high dimension case when $N < k$ but the $\beta$s to be recovered are sparse, the work in \cite{Buhlmann2010} proposes changing the vanilla EM for this problem, by adding a Lasso penalty to the $\beta$ update step. For this method, and sufficient samples, they show that there exists a local minimizer which selects the correct support. This can be viewed as an interesting extension of the known fact about EM, that it has efficient local minima, to the sparse case; however there are no guarantees that any (or even several) runs of this modified EM will actually {\em find} this good local minimum.

In recent years, an interesting line of work (e.g., \cite{Hsu2012Gussian}, \cite{Anandkumar2012Tensor}) has shown the possibility of resolving latent variable models via considering spectral properties of appropriate third-order tensors. Very recent work \cite{chaganty13} applies this approach to mixed linear regression. Their method suffers from high sample complexity; in the setting of our problem, their theoretical analysis indicates $N > O(k^6)$. Additionally, this method has much higher computational complexity than the methods in our paper (both EM, and the initialization), due to the fact that they need to work with third-order tensors. 

A quite similar problem that attracts extensive attention is subspace clustering, where the goal is to learn an unknown number of linear subspaces of varying dimensions from sample points. Putting our problem in this setting, each sample $(y,\vecfont{x})$ is a vector in $\mathcal{R}^{k+1}$; the points from $\beta_1$ correspond to one $k$-dimensional subspace, and from $\beta_2$ to another $k$-dimensional subspace. Note that this makes for a very hard instance of subspace clustering, as not only are the dimensions of each subspace very high (only one less than ambient), but the projections of the points in the first $k$ coordinates are exactly the same. Even without the latter restriction, one typical method \cite{vidal2003generalized}, \cite{Ehsan2012} -- as an example -- requires $N \geq \order{k^2}$ to have unique solution. 
%In a recent paper\cite{soltanolkotabi2013robust} an algorithm based on spectral clustering is proposed and shown to work in the sparse subspace case. Their scheme does not fit our problem due to the rigorous limit of subspace dimension.

\subsection{Notation}

For matrix $X$, we use $\sigma_i(X)$ to denote the $i$th singular value of $X$. We denote the spectral, or operator, norm by $\|X\| := \max_{i} \sigma_i(X)$. For any vector $\vecfont{x}$ and scalar $p$, $\|\vecfont{x}\|_p$ is defined as the usual $\ell_p$ norm. For two vectors $\vecfont{x},\vecfont{y}$ we use $\iprod{\vecfont{x}}{\vecfont{y}}$ to denote their inner product and $\vecfont{x}\otimes\vecfont{y}$ to denote their outer product. $\vecfont{x}^T$ is transpose of $\vecfont{x}$. We define $T(\vecfont{x},\vecfont{y})$ to be the subspace spanned by $\vecfont{x}$ and $\vecfont{y}$. The operator $\mathcal{P}_{T(\vecfont{x},\vecfont{y})}$ is the orthogonal projection on $T(\vecfont{x},\vecfont{y})$. We use $N$ denote number of sample. $k$ is dimension of unknown parameters. 

\section{Algorithms}

In this section we describe the classical EM algorithm as is applied to our problem of mixed linear regression, and our new initialization procedure. Since our analytical results are currently only for the noiseless case, we focus here on EM for this setting, even though EM and also our initialization procedure easily apply to the general setting. The iterations of EM involve alternating between {\em (a)} given current $\beta_1,\beta_2$, partitioning the samples into $J_1$ (which are more likely to have come from $\beta_1$) and $J_2$ (respectively, from $\beta_2$), and then {\em (b)} updating each of $\beta_1,\beta_2$ given the new sample sets $J_1,J_2$ corresponding to each, respectively. Both parts of the iteration are extremely efficient, and can be scaled easily to large problem sizes. In the typical application, in the absence of any extraneous side information, the initial $\beta^{(0)}$'s are chosen at random.

\begin{algorithm}[h]
\caption{EM (noiseless case)}
\label{alg:altmin}
\begin{algorithmic}[1]
\INPUT Initial $\vecfont{\beta}_1^{(0)}, \vecfont{\beta}_2^{(0)}$, \# iterations $t_0$, samples $\{(y_i, \vecfont{x}_i), i =1,2,...,N\}$
\FOR{$t = 0,\cdots,t_0-1$}
\STATE \COMMENT{{\em EM Part I: Guess the labels}}
\STATE $J_1, J_2 \leftarrow \emptyset$
\FOR{$i = 1, 2, \cdots, N$}
\IF{$\left|y_i - \iprod{\vecfont{x}_i}{\vecfont{\beta}_1^{(t)}}\right| < \left|y_i - \iprod{\vecfont{x}_i}{\vecfont{\beta}_2^{(t)}}\right|$} 
\STATE	$J_1 \leftarrow J_1 \cup \{i\}$
\ELSE
\STATE	$J_2 \leftarrow J_2 \cup \{i\}$
\ENDIF
\ENDFOR
\STATE \COMMENT{\em EM Part II: Solve least squares}
\STATE $\vecfont{\beta}_1^{(t+1)} \leftarrow \argmin_{\vecfont{\beta} \in \reals^k} \twonorm{\vecfont{y}_{J_1} - \mat{X}_{J_1} \vecfont{\beta}}$
\STATE $\vecfont{\beta}_2^{(t+1)} \leftarrow \argmin_{\vecfont{\beta} \in \reals^k} \twonorm{\vecfont{y}_{J_2} - \mat{X}_{J_2} \vecfont{\beta}}$
\ENDFOR
\OUTPUT $\vecfont{\beta}_1^{(t_0)}, \vecfont{\beta}_2^{(t_0)}$
\end{algorithmic}
\end{algorithm}

It is not hard to see that each iteration of the above procedure results in a decrease in the loss function
\begin{equation}
\mathcal{L}(\beta_1,\beta_2)  :=  \sum_i \, \min_{z_i\in \{0,1\}} \, \left ( y_i - \iprod{\vecfont{x}_i}{z_i \beta_1 + (1-z_i)\beta_2}   \right )^2.
\label{eq:loss}
\end{equation}
Note that $\mathcal{L}$, being the minimum of several convex functions, is neither convex nor concave; hence, while EM is guaranteed to converge, all that can be said {\it a priori} is that it will reach a local minimum. Indeed, our hardness result in Section \ref{sec:mainresults} confirms that for general $\mathbf{x}_i$, this must be the case. Yet even for the Gaussian case we consider, this has essentially been the state of analytical understanding of EM for this problem to date; in particular there are no global guarantees on convergence to the true solutions, under any assumptions, as far as we are aware. 

The main algorithmic innovation of our paper is to develop a more principled initialization procedure. In practice, this allows for faster convergence, and with fewer samples, to the true $\beta^*_1,\beta^*_2$. Additionally, it allows us to establish global guarantees for EM, when EM is started from here. We now describe this initialization.

\subsection{Initialization} 

Our initialization procedure is based on the positive semidefinite matrix
\[
M := \frac{1}{N}\sum_{i=1}^{N}y_i^2\vecfont{x}_i\otimes\vecfont{x}_i,
\]
where $\otimes$ represents the outer product of two vectors. The main idea is that $M$ is an unbiased estimator of a matrix whose top two eigenvectors span the same space spanned by the true $\beta_1^*,\beta_2^*$. We now present the idea, and then formally describe the procedure.

{\bf Idea:} The expected value of $M$ is given by
\[
\mathbb{E}[M] ~ = ~ p_1 A_1 + p_2 A_2,
\]
where $p_1,p_2$ are the fractions of observations of $\beta_1^*,\beta_2^*$ respectively, and the matrices $A_i$, $i=1,2$, are given by
\[
A_i ~ := ~ \mathbb{E}\left [ \, \iprod{\vecfont{x}}{\beta^*_i}^2 \, \vecfont{x}\otimes \vecfont{x}\right],
\] 
where the expectation is over the random vector $\vecfont{x}$, which in our setting is uniform normal. It is not hard to see that this matrix evaluates to
\[
A_i ~ = ~ I \, +\,  2 (\beta_i^* \otimes \beta_i^*)
\]
where $I$ is the identity matrix. Thus it has $\beta_i^*$ as its leading eigenvector (with eigenvalue $1 + 2 \|\beta_i^*\|^2$), and all other eigenvalues are 1. Thus, as long as neither of the fractions $p_1$, $p_2$ are too small, the leading eigenvectors of the expectation $\mathbb{E}[M]$ will be the true vectors $\beta^*_1,\beta^*_2$. Of course, we do not have access to this expected matrix; however, note that $M$ is the sum of i.i.d. matrices, and thus one can expect that with sufficient samples $N$, the top-2 eigenspace will be a decent approximation of the space spanned by $\beta_1^* ,\beta_2^*$. 

Note however that, {\em even} for the expected matrix $\mathbb{E}[M] $, when $p_1=p_2$ and $\|\beta_1^*\| = \|\beta_2^*\|$ (the case we argue is the most pertinent and difficult) the top two eigenvectors will {\em not} be $\beta^*_1,\beta^*_2$, since these two vectors need not be orthogonal. We thus need to run a simple 1-dimensional grid search on the unit circle in this space to find good approximations to the individual vectors $\beta^*_1,\beta^*_2$, as opposed to just the space spanned by them. Our algorithm uses the empirical loss of every candidate pair, $(\hat{\beta}_1,\hat{\beta}_2)$, produced by the grid search, in order to select a good initial starting point.

The details of the above idea are given below, along with the formal description of our procedure, in Algorithm \ref{alg:init}. 
%We first make the matrix $M$, find the space spanned by its top two eigenvectors, and then do a (1-dimensional) grid search on the unit circle in this space to find estimates for the $\beta_1,\beta_2$. In particular, we evaluate the empirical loss $\mathcal{L}$, as given in eq. (\ref{eq:loss}), for each pair of the grid points, and find the first one that has the lowest loss. 

\begin{algorithm}[h]
\caption{Initialization}
\label{alg:init}
\begin{algorithmic}[1]
\INPUT Grid resolution $\delta$, samples $\{(y_i, \vecfont{x}_i), i =1,2,...,N\}$
\STATE $M \leftarrow \frac{1}{N}\sum_{i=1}^{N}y_i^2\vecfont{x}_i\otimes\vecfont{x}_i$
\STATE Compute top 2 eigenvectors $\vecfont{v}_1, \vecfont{v}_2$ of $M$
\STATE {\em \{Make the grid points\}} \\
$G \leftarrow \{\vecfont{u}:\vecfont{u} = \vecfont{v}_1\cos(\delta t) + \vecfont{v}_2\sin(\delta t), t = 0,1,...,\lceil\frac{2\pi}{\delta}\rceil\}$
\STATE {\em \{Pick the pair that has the lowest loss\} } 
\[
\vecfont{\beta}_1^{(0)},\vecfont{\beta}_2^{(0)} \leftarrow \arg \min_{\vecfont{u}_1,\vecfont{u}_2 \in G} \mathcal{L}(\vecfont{u}_1,\vecfont{u}_2)
\]
\OUTPUT $\vecfont{\beta}_1^{(0)}, \vecfont{\beta}_2^{(0)}$
%
%$\mathcal{L}(\vecfont{u}_1,\vecfont{u}_2)$
%$\vecfont{u}_1,\vecfont{u}_2$ 
%
%\STATE $\hat{W} = \{(\vecfont{u},\vecfont{v}): \vecfont{u} \neq \vecfont{v},\vecfont{u} \in W, \vecfont{v} \in W \}$
%\STATE \textbf{while}($\hat{W}$ is nonempty) 
%\STATE $\;\;$ pick $(\vecfont{u}, \vecfont{v}) $ from $\hat{W}$ 
%\STATE $\;\;$ $\hat{W} \leftarrow \hat{W}/\{(\vecfont{u}, \vecfont{v})\}$
%\STATE $\;\;$ \textbf{if} $\sqrt{\mathcal{L}(\vecfont{u},\vecfont{v})/N} < 1.1\epsilon$ \textbf{break}
%\STATE \textbf{end while} 
%\STATE run \textsl{EM} with input $(\vecfont{u}, \vecfont{v})$, $t_0$.
%\OUTPUT $\vecfont{\beta}_1^{(t_0)}, \vecfont{\beta}_2^{(t_0)}$

\end{algorithmic}
\end{algorithm}

\textbf{Choice of grid resolution $\delta$.} In section \ref{sec:mainresults}, we show that it's sufficient to choose $\delta < c\|\vecfont{\beta}_1^* - \vecfont{\beta}_2^*\|_2 \sqrt{\min\{p_1,p_2\}}^3$ for some universal constant $c$. Even we have no knowledge of gound truth, successful choice of $\delta$ relies on a conservative estimation of $\|\vecfont{\beta}_1^* - \vecfont{\beta}_2^*\|_2$ and $\min\{p_1,p_2\}$. Note that {\em this upper bound does not scale with problem size}. The number of candidate pairs is actually independent of $(k,N)$. 

\textbf{Search avoidance method using prior knowledge of proportions.} When $p_1, p_2$ are known, approximation of $\vecfont{\beta}_1^*, \vecfont{\beta}_2^*$ can be computed from the top two eigenvectors of $M$ in closed form. Suppose $(\vecfont{v}_b^*, \lambda_b^*), b = 1,2$ are eigenvectors and eigenvalues of $\mathbb{E}[(M - I)/2] $. We define 
\begin{equation*}
sign(b) = \begin{cases}
1, \; b = 1\\
-1, \; b = 2 
\end{cases}
\end{equation*}
It is easy to check that when $\lambda_1^* \ne \lambda_2^*$ (we use $-b$ to denote $\{1,2\}\setminus b$),
\begin{equation}
\label{directAppro}
\vecfont{\beta}_b^* = \sqrt{\frac{1 - \Delta_b^*}{2}}\vecfont{v}_b^* + sign(b)\sqrt{\frac{1+\Delta_{b}^*}{2}}\vecfont{v}_{-b}^*, \; b = 1,2,
\end{equation}
where
\begin{equation*}
\Delta_b^* = \frac{(\lambda_{b}^* - \lambda_{-b}^*)^2 + p_b^2 - p_{-b}^2}{2(\lambda_{-b}^* - \lambda_b^*)p_b}, \; b = 1,2.
\end{equation*}

\textbf{Duplicate eigenvalues.} $\lambda_1^* = \lambda_2^*$ if and only if $p_1 = p_2$ and $\iprod{\vecfont{\beta}_1^*}{\vecfont{\beta}_2^*} = 0$. In this case $\{\vecfont{\beta}_1^*, \vecfont{\beta}_2^*\}$ are not identifiable from spectral structure of $\mathbb{E}(M)$ because any linear combination of $\{\vecfont{\beta}_1^*, \vecfont{\beta}_2^*\}$ is an eigenvector of $\mathbb{E}(M)$. We go back to Algorithm \ref{alg:init} in this case.

Based on the above analysis, we propose an alternative initialization method using proportion information when eigenvalues are nonidentical, in Algorithm \ref{alg:init_withp}.
\begin{algorithm}[h]
\caption{Initialization with proportion information}
\label{alg:init_withp}
\begin{algorithmic}[1]
\INPUT $p_1, p_2$, samples $\{(y_i, \vecfont{x}_i), i =1,2,...,N\}$
\STATE $M \leftarrow \frac{1}{N}\sum_{i=1}^{N}y_i^2\vecfont{x}_i\otimes\vecfont{x}_i$
\STATE Compute top 2 eigenvectors and eigenvalues $(\vecfont{v}_b, \lambda_b), b = 1,2$ of $(M-I)/2$
\STATE Compute $\vecfont{\beta}_1^{(0)}, \vecfont{\beta}_2^{(0)}$ via equation (\ref{directAppro}) ( use empirical version, i.e., remove superscript $*$)
\OUTPUT $\vecfont{\beta}_1^{(0)}, \vecfont{\beta}_2^{(0)}$

\end{algorithmic}
\end{algorithm}

In Section \ref{sec:simulations}, we demonstrate empirically the importance of this initialization technique; we show that EM initialized randomly has remarkably slower performance compared to EM initialized by Algorithm \ref{alg:init}. Our theoretical results presented in Section \ref{sec:mainresults}, confirm this observation analytically.

\section{Empirical Performance}
\label{sec:simulations}
In this section, we demonstrate the behavior of our algorithm on synthetic data. The results highlight in particular two important features of our results. First, the simulations corroborate our theoretical results given in Section \ref{sec:mainresults}, which show that our algorithm is nearly optimal (unimprovable) in terms of sample complexity. Indeed, we show here that EM+SVD succeeds when given about as many samples as dimensions (in the absence of additional structure, e.g., sparsity, it is not possible to do better). Second, our results show that the SVD initialization seems to be critical: without it, EM's performance is significantly degraded.

\textbf{Setting} We generate $\vecfont{x}$ from $\mathcal{N}(0,I)$. We then choose the labels uniformly at random, i.e., we set $p_1 = p_2 = 0.5$. Also, in each trial, we generate $\beta_1^*$ and $\beta_2^*$ randomly but keep $\iprod{\beta_1^*}{\beta_2^*} = 1.73$. This constant is arbitrarily chosen here. Our goal is to make sure they are non-orthogonal. We run algorithm \ref{alg:init} with a fairly coarse grid: $\delta = 0.3$. We also test algorithm \ref{alg:init_withp} using $p_1 = p_2$. The following metric which stands for global optimality is used
\begin{equation}
{\rm err}^{(t)} := \max\{\|\beta_1^{(t)} - \beta_1^*\|_2, \|\beta_2^{(t)} - \beta_2^*\|_2\}. \label{eq:error}
\end{equation}
Here $t$ is the sequence of number of iterations. 

\textbf{Sample Complexity.} In figure \ref{samplecomplexity} we empirically investigate how the number of samples $N$ needed for exact recovery scales with the dimension $k$. Each point in Figure \ref{samplecomplexity} represents 1000 trials, and the corresponding value of $N$ is the number of samples at which the success rate was greater than 0.99. We use algorithm \ref{alg:init} for initialization. In figure \ref{phasetransition}, we show the phase transition curves with a few $(N,k)$ pairs.

\begin{figure}[ht]
\vskip 0.2in
\begin{center}
%\centerline{\includegraphics[width=\columnwidth/2]{nsample.eps}}
\centerline{\includegraphics{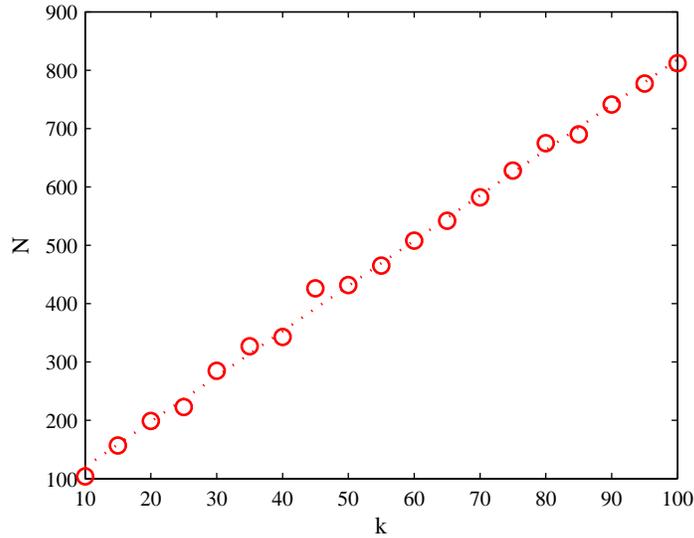}}
\caption{Number of samples needed for success rate greater than 0.99 using SVD+EM. The dotted line is the least square fit of the experimental data.}
\label{samplecomplexity}
\end{center}
\vskip -0.2in
\end{figure}

\begin{figure}[ht]
\vskip 0.2in
\begin{center}
%\centerline{\includegraphics[width=\columnwidth/2]{phaseTran.eps}}
\centerline{\includegraphics[scale = 0.8]{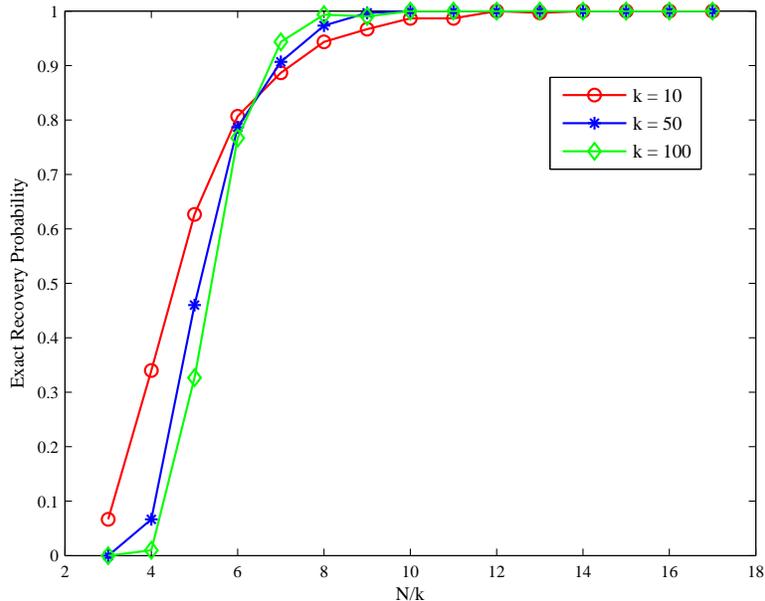}}
\caption{Success probability vs. normalized number of samples, i.e., $N/k$. }
\label{phasetransition}
\end{center}
\vskip -0.2in
\end{figure}

\textbf{Effect of Initialization.} We compare our eigenvector-based Initialization + EM with the usual randomly initialized EM.  For $N = 300$ samples and $k = 10$ dimensions, figure \ref{con1} shows how the error $err$ converges as a function of the iterations. Each curve is averaged over 200 trials.  We observe that the final error of SVD+EM is about $10^{-35}$. The level of noise results from float computation. For each trial, the blue and green curves show that exact recovery occurred after 7 iterations. This is possible since we are in the noiseless case. 

As can be clearly seen, initialization has a profound effect on the performance of EM in our setting; it allows for exact recovery with high probability in a small number of iterations, while random initialization does not. 

\begin{figure}[ht]
\vskip 0.2in
\begin{center}
%\centerline{\includegraphics[width=\columnwidth/2]{SVDvsRND.eps}}
\centerline{\includegraphics{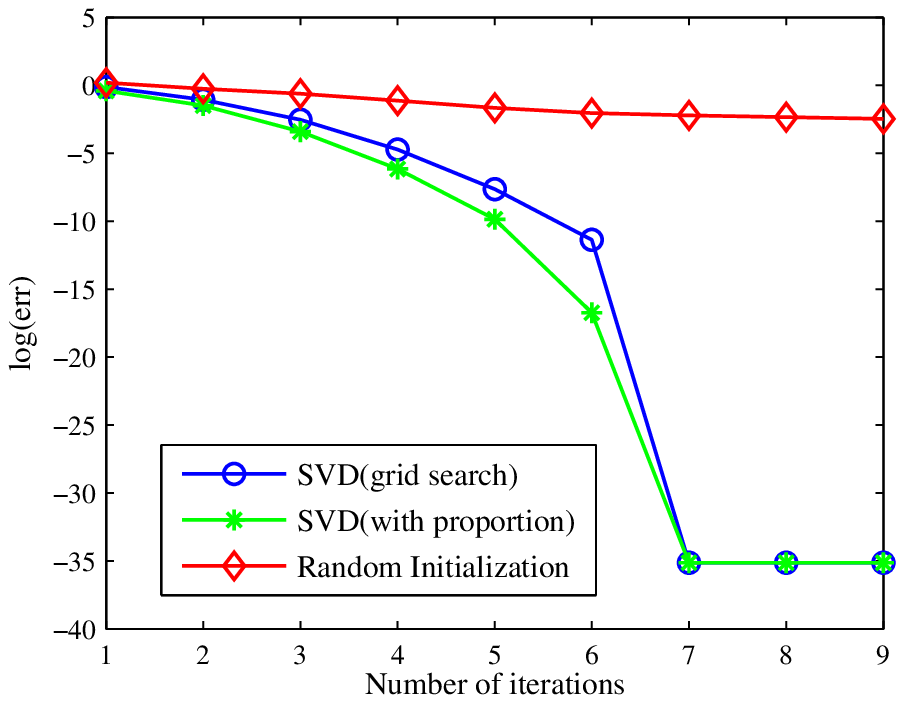}}
\caption{This figure compares the decay in error, as a function of iteration count of EM, with and without our initialization. As can be seen, initialization allows for exact recovery (the $10^{-35}$ error is precision of Matlab) in a small number of iterations, while the standard random initialization is still not close.}
\label{con1}
\end{center}
\vskip -0.2in
\end{figure}

%We then set $w_i \sim \mathcal{N}(0,0.01)$ and run the above simulation again. Figure \ref{con2} shows SVD is an effective initial step even in noisy case. 
%
%\begin{figure}[ht]
%\vskip 0.2in
%\begin{center}
%\centerline{\includegraphics[width=\columnwidth]{SVDvsRNDnoise.eps}}
%\caption{Convergence of SVD+EM and EM with random initialization(noisy case)}
%\label{con2}
%\end{center}
%\vskip -0.2in
%\end{figure}

\section{Main Results}
\label{sec:mainresults}

In this section, we present the main results of our paper: provable statistical guarantees for EM, initialized with our Algorithm \ref{alg:init}, in solving the mixed linear regression problem. We first show that for general $\{\mathbf{x}_i\}$, the problem is NP-hard, even without noise. Then, we focus on the setting where each measurement vector $\vecfont{x}_i$ is iid and sampled from the uniform normal distribution $\mathcal{N}(0,I)$. We also assume that the true vectors $\beta_1^*,\beta_2^*$ are equal in magnitude, which without loss of generality, we assume is 1. Intuitively, equal magnitudes represents a hard case, as in this setting the $y_i$'s from the two $\beta$'s are statistically identical\footnote{In particular, each $y_i$ has mean 0, and  variance $\|\beta_1^*\|^2$ if it comes from the first vector, and $\|\beta_2^*\|^2$ if it comes from the second. Having them be equal, i.e. $\|\beta_1^*\|^2 = \|\beta_2^*\|^2$, makes the $y_i$s statistically identical.}. 

Our proof can be broken into two key results. We first show that using $O(k\log^2 k)$ samples, with high probability our initialization  procedure returns $\beta_1^{(0)},\beta_2^{(0)}$ which are within a constant distance of the true $\beta_1^*,\beta_2^*$. We note that for our scaling guarantees to hold, this constant need only be independent of the dimension, and in particular, it need not depend on the final desired precision. Results with a $1/{\rm error}$ or even $1/{\rm error}^2$ dependence -- as would be required in order for the SVD step alone to obtain an approximation of $\beta_i^*$, $i=1,2$, to within some error tolerance, are exponentially worse than what our two-step algorithm guarantees. 

We then show that, given this good initialization, at any subsequent step $t$ with current estimate $(\beta_1^{(t)},\beta_2^{(t)})$, doing one step of the EM iteration with samples that are independent of these $\beta_i^{(t)}$ results in the error decreasing by a factor of half, hence implying geometric convergence. As we explain below, our analysis providing this guarantee depends on using a new set of samples, i.e., the analysis does not allow re-use samples across iterations, as typically done in EM. We believe this is an artifact of the analysis; and of course, in practice, reusing the samples in each iteration seems to be advantageous.

Thus, our analytical results are for {\em resampled versions} of EM and the initialization scheme, which we state as Algorithms \ref{alg:rspEM} and \ref{alg:rspINIT} below. Essentially, resampling involves splitting the set of samples into disjoint sets, and using one set for each iteration of EM; otherwise the algorithm is identical to before.  Since we have geometric decrease in the error, achieving an $\epsilon$ accuracy comes at an additional cost of a factor $\log(1/\epsilon)$ in the sample complexity, as compared to what may have been possible with the non-resampled case. We then show that when $\epsilon \leq \order{1/k^2}$, the error decays to be zero with high probability. In other words, we need in total $\order{\log k}$ iterations in order to do exact recovery. Additionally, and the main contribution of this paper, the resampled version given here, represents the only known algorithm, EM or otherwise, with provable global statistical guarantees for the mixed linear regression problem, with sample complexity close to $O(k)$.

\begin{algorithm}[h]
\caption{EM with resampling}
\label{alg:rspEM}
\begin{algorithmic}[1]
\INPUT Initial $\vecfont{\beta}_1^{(0)}, \vecfont{\beta}_2^{(0)}$, \# iterations $t_0$, samples $\{(y_i, \vecfont{x}_i), i =1,2,...,N\}$
\STATE Partition the samples $\{(y_i, \vecfont{x}_i)\}$ into $t_0$ disjoint sets: $\mathcal{S}_1,...,\mathcal{S}_{t_0}$.
\FOR{$t = 1,\cdots,t_0$}
\STATE Use $\mathcal{S}_t$ to run lines \textsl{2} to \textsl{13} in algorithm \ref{alg:altmin}.
\ENDFOR
\OUTPUT $\vecfont{\beta}_1^{(t_0)}, \vecfont{\beta}_2^{(t_0)}$
\end{algorithmic}
\end{algorithm}

Similarly, in the initialization procedure, for analytical guarantees we require two separate sets of samples: one set $\mathcal{S}_{*}$ for finding the top-2 eignespace, and another set $\mathcal{S}_+$ for evaluating the loss function for grid points. % This is given below.

\begin{algorithm}[h]
\caption{Initialization with resampling}
\label{alg:rspINIT}
\begin{algorithmic}[1]
\INPUT Grid resolution $\delta$, samples $\{(y_i, \vecfont{x}_i), i =1,2,...,N\}$
\STATE Partition the samples $\{(y_i, \vecfont{x}_i)\}$ into two disjoint sets: $\mathcal{S}_{*}, \mathcal{S}_+$
\STATE $M \leftarrow \frac{1}{|\mathcal{S}_*|}\sum_{i\in \mathcal{S}_*}y_i^2\vecfont{x}_i\otimes\vecfont{x}_i$
\STATE Compute top 2 eigenvectors $\vecfont{v}_1, \vecfont{v}_2$ of $M$
\STATE {\em \{Make the grid points\}} \\
$G \leftarrow \{\vecfont{u}:\vecfont{u} = \vecfont{v}_1\cos(\delta t) + \vecfont{v}_2\sin(\delta t), t = 0,1,...,\lceil\frac{2\pi}{\delta}\rceil\}$
\STATE {\em \{Pick the pair that has the lowest loss\} } 
\[
\vecfont{\beta}_1^{(0)},\vecfont{\beta}_2^{(0)} \leftarrow \arg \min_{\vecfont{u}_1,\vecfont{u}_2 \in G} \mathcal{L}_+(\vecfont{u}_1,\vecfont{u}_2)
\]
where this loss $\mathcal{L}_+$ is evaluated as in (\ref{eq:loss}) using samples in $\mathcal{S}_+$
\OUTPUT $\vecfont{\beta}_1^{(0)}, \vecfont{\beta}_2^{(0)}$
\end{algorithmic}
\end{algorithm}

First, we provide the hardness result for the case of general $\{\mathbf{x}_i\}$.

\begin{proposition}\label{prop:hardness}
Deciding if a general instance of the mixed linear equations problem specified by $(\mathbf{y},X)$ has a solution, $\beta_1, \beta_2$, is NP-hard.
\end{proposition}
The proof follows via a reduction from the so-called {\sc SubsetSum} problem, which is known to be NP-hard\cite{GareyJohnson}. We postpone the details to the supplemental material.

We now state two theoretical guarantees of the initialization algorithms. Recall that the error $err^{(t)}$ is as given in (\ref{eq:error}), and $p_1,p_2$ are the fractions of observations that come from $\beta_1^*,\beta_2^*$ respectively. 

The following result guarantees a good initialization (algorithm \ref{alg:rspINIT}) {\em without requiring sample complexity that depends on the final target error of the ultimate solution}. Essentially, it says that we obtain an initialization that is {\em good enough} using $O(k \log^2k)$ samples.

\begin{proposition} \label{thm:init}
Given any constant $\widehat{c} < 1/2$, with probability at least $1-c_3 k^{-2}$ Algorithm \ref{alg:rspINIT} produces an initialization $(\beta_1^{(0)},\beta_2^{(0)})$, satisfying 
\begin{align*}
err^{(0)} ~ \leq ~ \widehat{c} \, \min\{p_1,p_2\} \, \|\vecfont{\beta}_1^* - \vecfont{\beta}_2^*\|_2, 
\end{align*}
as long as we choose grid resolution $\delta \leq  \frac{2}{11} \widehat{c} \|\vecfont{\beta}_1^* - \vecfont{\beta}_2^*\|_2  \sqrt{\min\{p_1, p_2\}}^3$, and the number of samples $|\mathcal{S}_{*}|$ and $|\mathcal{S}_{+}|$ satisfy:
%Consider Algorithm \ref{alg:rspINIT}, and suppose we are given a constant $\widehat{c} < 1/2$.  If we choose grid resolution $\delta\leq  \frac{4}{11} \widehat{c} \sqrt{\min\{p_1, p_2\}}\|\vecfont{\beta}_1^* - \vecfont{\beta}_2^*\|_2$ , then there exist constants $c_1,c_2,c_3$ so that  if the number of samples satisfies
\begin{eqnarray*}
|\mathcal{S}_{*}| & \geq & c_1 \left (\frac{1}{\widetilde{\delta}}\right )^2 \, k \, \log^2 k \\
|\mathcal{S}_{+}| & \geq & \left(\frac{c_2}{\min\{p_1,p_2\}}\right) \, k,
\end{eqnarray*}
where $c_1$, $c_2$ and $c_3$ depend on $\hat{c}$ and $\min\{p_1,p_2\}$ but not on the dimension, $k$, and where
\[
\widetilde{\delta} = \frac{\delta^2}{384}(1 - \sqrt{1-4(1-\iprod{\vecfont{\beta}_1^*}{\vecfont{\beta}_2^*}^2)p_1p_2}).
\] 
%then with probability at least $1-c_3 k^{-2}$ the error of the output $\beta_1^{(0)},\beta_2^{(0)}$ will satisfy 
%\begin{align*}
%err^{(0)} ~ \leq ~ \widehat{c} \, \min\{p_1,p_2\} \, \|\vecfont{\beta}_1^* - \vecfont{\beta}_2^*\|_2. 
%\end{align*}
\end{proposition}

%We note that while the grid resolution depends on $\min\{p_1,p_2\}$, as discussed above, the most difficult instance of the mixed regression problem is when the labels are nearly balanced, i.e., $p_1 \approx p_2$; for otherwise, assigning all labels to be identical would already be a good initialization of the problem. 

Algorithm \ref{alg:init_withp} can be analyzed without resampling argument. The input sample set is $\mathcal{S}_{*}$, we have the following conclusion.
\begin{proposition} \label{prop:initwp}
Consider initialization method in algorithm \ref{alg:init_withp}. Given any constant $\widehat{c} < 1/2$, with probablity at least $1 - \frac{1}{k^2}$, the approach produces an initialization $(\beta_1^{(0)},\beta_2^{(0)})$ satisfying
\begin{align*}
err^{(0)} ~ \leq ~ \widehat{c} \, \min\{p_1,p_2\} \, \|\vecfont{\beta}_1^* - \vecfont{\beta}_2^*\|_2, 
\end{align*}
if 
\begin{eqnarray*}
|\mathcal{S}_{*}| & \geq & c_1 \left (\frac{1}{\widetilde{\delta}}\right )^2 \, k \, \log^2 k. 
\end{eqnarray*}
Here $c_1$ is a constant that depends on $\widehat{c}$. And
\[
\sqrt{\widetilde{\delta}} = \widehat{c}\sqrt{\min\{p_1,p_2\}}^3\|\beta_1^* - \beta_2^*\|_2(\sqrt{1 - \kappa})\kappa,
\]
where $\kappa = \sqrt{1-4(1-\iprod{\vecfont{\beta}_1^*}{\vecfont{\beta}_2^*}^2)p_1p_2}$.
\end{proposition}

Comparing the obtained upper bound of $\widetilde{\delta}$ with that in proposition \ref{thm:init}, we note there is an additional $\kappa$ factor. Actually, $\kappa$ represents the gap between top two eigenvectors of $\mathbb{E}(M)$. This factor characterizes the hardness of identifying two vectors from search avoiding method.  

The proofs of proposition \ref{thm:init} and \ref{prop:initwp} relies on standard concentration results and eigenspace perturbation analysis. We postpone the details to supplemental materials.

The main theorem of the paper guarantees geometric decay of error, assuming a good initialization. Essentially, this says that to achieve error less than $\epsilon$, we need $\log (1/\epsilon)$ iterations, each using $O(k)$ samples. Again, we note the absence of higher order dependence on the dimension, $k$, or anything other than the mild dependence on the final error tolerance, $\epsilon$.

\begin{theorem} \label{thm:em}
Consider one iteration in algorithm \ref{alg:rspEM}. For fixed $(\beta_1^{(t-1)},\beta_2^{(t-1)})$, there exist absolute constants $\widetilde{c},c_1,c_2$ such that if
\begin{equation*}
\label{cond_err}
  err^{(t-1)} ~ \leq  ~ \widetilde{c} \, \min \{p_1,p_2\} \,  \|\mathbf{\beta}_1^* - \mathbf{\beta}_2^*\|_2,
\end{equation*}
and if the number of samples in that iteration satisfies
\[
|\mathcal{S}_t| ~ \geq ~ \left( \frac{c_1}{ \min \{p_1,p_2\}} \right) \, k,
\]
then with probability greater than $1-\exp(-c_2 k)$ we have a geometric decrease in the error at the next stage, i.e. 
\begin{equation*}
 err^{(t)} \leq \frac{1}{2} err^{(t-1)}
\end{equation*}
\end{theorem}

Note that the decrease factor $1/2$ is arbitrarily chosen here. To put the above results together, we choose the constant $\widehat{c}$ in proposition \ref{thm:init} and \ref{prop:initwp} to be less than the constant $\widetilde{c}$ in Theorem \ref{thm:em}. Then, in each iteration of alternating minimization, with $\order{k}$ fresh samples, the error decays geometrically by a constant factor with probability greater than $1 - \exp{-ck}$. Suppose we are satisfied with error level $\epsilon$, resampling regime requires $\order{k\log^2k + k\log(1/\epsilon)}$ number of samples. 

Let $J_b^*$ denote the set of samples generated from $\vecfont{\beta}_b^*, b =1,2$. It's not hard to observe that in noiseless case, exact recovery occurs when $J_b = J_b^*$.  The next result shows that when $\epsilon < \frac{c}{k^2}\|\vecfont{\beta}_1^* - \vecfont{\beta}_2^*\|_2$, fresh $\Theta(k)$ samples will be clustered correctly which results in exact recovery.

\begin{proposition}(Exact Recovery)\label{prop:exactReco}
There exist absolute constants $c_1,c_2$ such that if 
\[
err^{(t-1)} \leq \frac{c_1}{k^2}\|\vecfont{\beta}_1^* - \vecfont{\beta}_2^* \|_2
\]
and 
\[
\frac{1}{\min\{p_1,p_2\}}k < |\mathcal{S}_t| < c_2 k,
\]
then with probability greater than $1 - \frac{1}{k}$,
\[
err^{(t)} = 0.
\]
\end{proposition}
By setting $\epsilon = \order{1/k^2}$, it turns out that exact recovery needs totally $\order{k\log^2k}$ samples. On using alternating minimization, approximation error will decay geometrically in the first place. Then when error hits some level, exact recovery occurs and the ground truth is found. Simulation results in figure \ref{con1} supports our conclusion.

\section{Proofs}
In this section, we provide the proofs of our two main results: we first show that the initializations produces an initial starting point $(\hat{\beta}_1^{(0)},\hat{\beta}_2^{(0)})$ that is within constant distance away from the truth (proposistions \ref{thm:init} and \ref{prop:initwp}). We then show that with a good starting point, EM exhibits geometric convergence, reducing the error by a factor of $2$ at each iteration (theorem \ref{thm:init}).

We postpone the proofs of proposition \ref{prop:hardness}, proposition \ref{prop:exactReco} and a few technical supporting lemmas to the appendix.

\subsection{Proof of Proposition \ref{thm:init}}
To show that our SVD initialization produces a good initial solution, requires two steps. Recall that Algorithm \ref{alg:rspINIT} finds the two dimensional subspace spanned by the top two eigenvectors of the matrix $M = \frac{1}{|\mathcal{S}_*|} \sum_{i\in \mathcal{S}_*} y_i^2 \mathbf{x}_i \otimes \mathbf{x}_i$, and then searches on a discretization of the circle in that subspace for two vectors that minimize the loss function, $\mathcal{L}_+$ evaluated on the samples in $\mathcal{S}_+$. 

We first show that the top eigenspace of $M$ is indeed close to the top eigenspace of its expectation, $p_1\vecfont{\beta}_1^*\otimes \vecfont{\beta}_1^* + p_2\vecfont{\beta}_2^*\otimes \vecfont{\beta}_2^* + I$, i.e., it is close to ${\rm span}\{\beta_1^*,\beta_2^*\}$, and that some pair of elements of the discretization are close to $(\beta_1^*,\beta_2^*)$. This is the content of lemma \ref{lem:concentration}. We then show that our loss function $\mathcal{L}_+$ is able to select good points from the discretization.

%First, we must show that the
%We have show the expectation of $M$ is $p_1\vecfont{\beta}_1^*\otimes \vecfont{\beta}_1^* + p_2\vecfont{\beta}_2^*\otimes \vecfont{\beta}_2^* + I$. We first present a result that shows it's highly possible to find good approximate vectors ($\epsilon$ away from the truth) from grid points $W$.

Our algorithm then uses the loss function $\mathcal{L}_+$ (evaluated on new samples in $\mathcal{S}_+$) to select good points from the grid $G$. Lemma \ref{sim} shows that as long as the number $\mathcal{S}_+$ of these new samples is large enough, we can upper {\em and} lower bound, with high probability, the empirically evaluated loss $\mathcal{L}_+(\hat{\vecfont{\beta}}_1,\hat{\vecfont{\beta}}_2)$ of any candidate pair $\hat{\vecfont{\beta}}_1,\hat{\vecfont{\beta}}_2$ by the true error {\rm err} of that candidate pair. This provides the critical result allowing us to do the correct selection in the 1-d search phase.

Now we are ready to prove the result. Suppose the conditions of lemma \ref{lem:concentration} hold. Then we are guaranteed the existence of $(\bar{\beta}_1,\bar{\beta}_2)$ in the grid $G$ with $\delta$-resolution, such that $\max_i \|\bar{\beta}_i - \beta_i^*\| < \delta$. Next, let $(\vecfont{\beta}_1^{(0)}, \vecfont{\beta}_2^{(0)})$ be the output of our SVD initialization, and let ${\rm err}$ denote their distance from $(\beta_1^*,\beta_2^*)$. By definition, the vectors $(\vecfont{\beta}_1^{(0)}, \vecfont{\beta}_2^{(0)})$ minimize the loss function $\mathcal{L}_+$ taken on inputs $\mathcal{S}_+$, and hence $\mathcal{L}_+(\vecfont{\beta}_1^{(0)}, \vecfont{\beta}_2^{(0)}) \leq \mathcal{L}_+(\bar{\beta}_1,\bar{\beta}_2)$. Using the lower bound from lemma \ref{sim}, applied to $(\vecfont{\beta}_1^{(0)}, \vecfont{\beta}_2^{(0)})$ we have:
\begin{equation*}
\frac{1}{5}\sqrt{\min\{p_1,p_2\}}{\rm err} \leq \sqrt{\frac{\mathcal{L}_+(\vecfont{\beta}_1^{(0)}, \vecfont{\beta}_2^{(0)})}{|\mathcal{S}_+|}}.
\end{equation*}
From the upper bound applied to $(\bar{\beta}_1,\bar{\beta}_2)$, we have
\begin{equation*}
\sqrt{\frac{\mathcal{L}_+(\bar{\beta}_1,\bar{\beta}_2)}{|\mathcal{S}_+|}} \leq 1.1 \delta.
\end{equation*}
Recalling that $\mathcal{L}_+(\vecfont{\beta}_1^{(0)}, \vecfont{\beta}_2^{(0)}) \leq \mathcal{L}_+(\bar{\beta}_1,\bar{\beta}_2)$, and taking 
\begin{equation*}
\delta \leq \frac{2}{11}\widehat{c}\|\vecfont{\beta}_1^* - \vecfont{\beta}_2^*\|_2\sqrt{\min\{p_1,p_2\}}^3,
\end{equation*}
we combine to finally obtain:
\begin{align*}
{\rm err} &\leq \frac{11}{2}\frac{\delta}{\sqrt{\min\{p_1,p_2\}}} \\
& \leq \widehat{c} \min\{p_1,p_2\}\|\vecfont{\beta}_1^* - \vecfont{\beta}_2^*\|_2.
\end{align*}
where $\widehat{c}$ is as in the statement of proposition \ref{thm:init}.

\subsection{Proof of Proposition \ref{prop:initwp}}
Using standard concentration results, in lemma \ref{lem:concentration}, we have shown if 
\[
|\mathcal{S}_*| > c(1/\widetilde{\delta})^2k\log^2k,
\]
with probability at least $1 - \frac{1}{k^2}$,
\[
\|M - \mathbb{E}(M)\| < 3\widetilde{\delta}
\]
Hence, we have 
\[
\big| |\lambda_1^* - \lambda_2^*| - |\lambda_1 - \lambda_2| \big| \leq 6\widetilde{\delta}.
\]
The approximate error of $\Delta_b^*$ can be bounded as:
\begin{align*}
2p_b|\Delta_b^* - \Delta_b| & \leq 6\widetilde{\delta} + (p_b^2 - p_{-b}^2)[\frac{1}{\lambda_{-b}^* - \lambda_b^*} - \frac{1}{\lambda_{-b} - \lambda_b}] \\
& \leq 6\widetilde{\delta} + |p_b^2 - p_{-b}^2| \frac{6\widetilde{\delta}}{(\lambda_{-b}^* - \lambda_b^*)(\lambda_{-b} - \lambda_b)} \\
& \leq 6\widetilde{\delta} + |p_b^2 - p_{-b}^2| \frac{6\widetilde{\delta}}{|\lambda_{-b}^* - \lambda_b^*|(|\lambda_{-b}^* - \lambda_b^*| - 6\widetilde{\delta})} \\
& \leq 6\widetilde{\delta} + |p_b^2 - p_{-b}^2| \frac{12\widetilde{\delta}}{|\lambda_{-b}^* - \lambda_b^*|^2} \\
\end{align*}
In the last inequality we use $\widetilde{\delta} \leq \frac{|\lambda_1^* - \lambda_2^*|}{12}$.

Next, we calculate approximation error of eigenvectors. Note that $\mathbb{E}(\frac{M - I}{2}) = p_1 \vecfont{\beta}_1^* \otimes \vecfont{\beta}_1^* + p_2 \vecfont{\beta}^*_2 \otimes \vecfont{\beta}^*_2$, we have
\[
\{\lambda_1^*, \lambda_2^*\} = \{\frac{1+\kappa}{2}, \frac{1-\kappa}{2} \}.
\]
Using lemma \ref{MatBound}, we have,
\[
\|\vecfont{v}_b - \vecfont{v}_b^*\|_2^2 \leq \frac{6\widetilde{\delta}}{\kappa} + \frac{24\widetilde{\delta}}{1 - \kappa} \leq \frac{24\widetilde{\delta}}{\kappa(1 - \kappa)}, \; b = 1,2.
\]
Then 
\begin{equation}
\label{errV}
\|\vecfont{\beta}_b^* - \vecfont{\beta}_b\|_2 \leq \bigg| \sqrt{\frac{1 - \Delta_b^*}{2}}\vecfont{v}_b^* - \sqrt{\frac{1 - \Delta_b}{2}}\vecfont{v}_b\bigg| + \bigg| \sqrt{\frac{1 + \Delta_{b}^*}{2}}\vecfont{v}_{-b}^* - \sqrt{\frac{1 + \Delta_{b}}{2}}\vecfont{v}_{-b} \bigg|.
\end{equation}
Note that
\begin{align*}
\bigg| \sqrt{\frac{1 - \Delta_b^*}{2}}\vecfont{v}_b^* - \sqrt{\frac{1 - \Delta_b}{2}}\vecfont{v}_b\bigg| & = \sqrt{\frac{1 - \Delta_b^*}{2}}\vecfont{v}_b^* - \sqrt{\frac{1 - \Delta_b^*}{2}}\vecfont{v}_b + \sqrt{\frac{1 - \Delta_b^*}{2}}\vecfont{v}_b - \sqrt{\frac{1 - \Delta_b}{2}}\vecfont{v}_b\bigg| \\
& \leq \sqrt{\frac{1 - \Delta_b^*}{2}} \|\vecfont{v}_b - \vecfont{v}_b^*\|_2 + \bigg|\sqrt{\frac{1 - \Delta_b^*}{2}} - \sqrt{\frac{1 - \Delta_b}{2}} \bigg| \|\vecfont{v}_b\|_2 \\
& \leq \|\vecfont{v}_b - \vecfont{v}_b^*\|_2 + \bigg|\sqrt{\frac{1 - \Delta_b^*}{2}} - \sqrt{\frac{1 - \Delta_b}{2}} \bigg| \\
& \leq \|\vecfont{v}_b - \vecfont{v}_b^*\|_2 + \sqrt{\frac{1}{2}\bigg| \Delta_b - \Delta_b^*\bigg|}.
\end{align*}
Plug the above result back to (\ref{errV}), we obtain
\begin{align*}
\|\vecfont{\beta}_b^* - \vecfont{\beta}_b\|_2 & \lesssim  \sqrt{\big| \Delta_b - \Delta_b^* \big|}  + \sum_{b} \|\vecfont{v}_b - \vecfont{v}_b^*\|_2 \\
& \lesssim \sqrt{\frac{\widetilde{\delta}}{\kappa(1 - \kappa)}} + \frac{1}{\sqrt{\min\{p_1,p_2\}}} \sqrt{\widetilde{\delta} + \frac{\widetilde{\delta}}{\kappa^2}} \\
& \lesssim \sqrt{\frac{\widetilde{\delta}}{\min\{p_1,p_2\}}} \times \sqrt{\frac{1}{\kappa(1 - \kappa)} + \frac{1}{\kappa^2}}  \\
& = \sqrt{\frac{\widetilde{\delta}}{\min\{p_1,p_2\}}} \frac{1}{\kappa\sqrt{1 - \kappa}}.
\end{align*}
By setting the above upper bound to be less than $\widehat{c}\min\{p_1,p_2\}\|\vecfont{\beta}_1^* - \vecfont{\beta}_2^*\|_2$, we complete the proof.

\subsection{Proof of Theorem \ref{thm:em}}
The following lemma is crucial.
\begin{lemma} \label{cd}
Assume $\mathbf{x} \in \mathbb{R}^k$ is a standard normal random vector. Let $u,v$ be two fixed vectors in $\mathbb{R}^k$. Define $\alpha_{(u,v)}:=\cos^{-1}\frac{(v-u)^{\top}(v+u)}{\|u+v\|_2\|u-v\|_2}$, $\alpha_{(u,v)} \in [0,\pi] $. Let $\Sigma = \mathbb{E}(\mathbf{x}\mathbf{x}^{\top}|(\mathbf{x}^{\top}u)^2 > (\mathbf{x}^{\top}v)^2 )$. Then,

(1)
\begin{equation} 
\sigma_{\max}(\Sigma ) = 1 + \frac{\sin\alpha_{(u,v)}}{\alpha_{(u,v)}},
\end{equation}
\begin{equation} 
\sigma_{\min}(\Sigma ) = 1 - \frac{\sin\alpha_{(u,v)}}{\alpha_{(u,v)}},
\end{equation}

(2)
\begin{equation}
\prob{(\mathbf{x}^{\top}u)^2 > (\mathbf{x}^{\top}v)^2} \left\{ 
\begin{aligned}
> & \frac{1}{2} &  \: \|u\|_2 > \|v\|_2 \\
\leq &  \frac{\|u\|_2}{\|v\|_2} &  \: \|u\|_2 < \|v\|_2
\end{aligned} 
\right.
\end{equation}
\end{lemma}

To simplify notation, we drop the iteration index $t$, and let $(\vecfont{\beta}_1, \vecfont{\beta}_2)$ denote the input to the EM algorithm, and  $(\vecfont{\beta}_1^+, \vecfont{\beta}_2^+)$ denote its output. Similarly, we write ${\rm err} := \max_i \|\beta_i - \beta_i^*\|$ and ${\rm err}^+ := \max_i \|\beta^+_i - \beta_i^*\|$. We denote by $J_1^*$ and $J_2^*$ the sets of samples that come from $\beta_1^*$ and $\beta_2^*$ respectively, and similarly we denote the sets produced by the ``E'' step using the current iteration $(\vecfont{\beta}_1, \vecfont{\beta}_2)$ by $J_1$ and $J_2$. Thus we have:
\begin{equation*}
J_1^* := \{i \in \mathcal{S}_t: y_i = \vecfont{x}_i^{\top}\vecfont{\beta}_1^*\},
\end{equation*}
and
\begin{equation*}
J_1 := \{i \in \mathcal{S}_t: (y_i - \vecfont{x}_i^{\top}\vecfont{\beta}_1)^2 < (y_i - \vecfont{x}_i^{\top}\vecfont{\beta}_2)^2 \},
\end{equation*}
and similarly for $J_2^*$ and $J_2$.

We define a diagonal matrix $W \in \mathbb{R}^{\mathcal{S}_t \times \mathcal{S}_t}$ to pick out the rows in $J_1$ when used for left multiplication: to this end, let $W_{ii} = 1$ if $i \in J_1$, and zero otherwise. Let $W^*$ be defined similarly, using $J_1^*$. Thus, $\beta_1^+$ is the least squares solution to $W\vecfont{y} = WX\vecfont{\beta}$, and $\beta_2^+$ is the least squares solution to $(I-W)\vecfont{y} = (I-W)X\vecfont{\beta}$, and
\begin{equation*}
\vecfont{y} = W^*X\vecfont{\beta}_1^* + (I - W^*)X\vecfont{\beta}_2^*.
\end{equation*}
Observing that $W^2 = W$, we have that $\vecfont{\beta}_1^+$ has closed form
\begin{equation*}
\vecfont{\beta}_1^+ = (X^{\top}WX)^{-1}X^{\top}W\vecfont{y}.
\end{equation*}
By simple algebraic calculation, we find
\begin{equation*}
\vecfont{\beta}_1^+ - \vecfont{\beta}_1^*= (X^{\top}WX)^{-1} X^{\top}(WW^*-W)X(\vecfont{\beta}_1^* - \vecfont{\beta}_2^*).
\end{equation*}
In order to bound the magnitude of the error and hence of the right hand side, we write
\begin{equation}
\label{12}
\|\vecfont{\beta}_1^+ - \vecfont{\beta}_1^*\|_2 \leq AB,
\end{equation}
where 
\begin{eqnarray*}
A &=& \|(X^{\top}WX)^{-1}\| \\
B &=& \big\|X^{\top}(W - WW^*)X(\vecfont{\beta}_1^* - \vecfont{\beta}_2^*)\big\|_2.
\end{eqnarray*}
\textbf{Bounding $A$.}
Observe that $X^{\top}WX = \sum_{i \in J_1} \vecfont{x}_i\vecfont{x}_i^{\top}$. Decomposing $J_1 = (J_1\cap J_1^*)  \cup (J_1\cap J_2^*)$, we have 
\begin{equation*}
\sigma_{\min}(X^{\top}WX) \geq \sigma_{\min}(\sum_{i \in J_1\cap J_1^*} \vecfont{x}_i\vecfont{x}_i^{\top}).
\end{equation*}
We need to control this quantity. We do so by lower bounding the number of terms in $J_1 \cap J_1^*$, and also the smallest singular value of the matrix $\Sigma =\expec{\{\vecfont{x}_i\vecfont{x}_i^{\top}| i \in J_1\cap J_1^* \}}$. 

If the current error satisfies
\begin{equation}
\label{errcond1}
{\rm err} \leq \frac{\|\vecfont{\beta}_1^* - \vecfont{\beta}_2^*\|_2}{2},
\end{equation}
we have $\|\vecfont{\beta}_1^* - \vecfont{\beta}_2\|_2 > \|\vecfont{\beta}_1^* - \vecfont{\beta}_1\|_2$. Now, from Lemma \ref{cd}, we have
\begin{equation*}
\prob{ (\vecfont{x}_i^{\top}(\vecfont{\beta}_1^* - \vecfont{\beta}_1))^2 < (\vecfont{x}_i^{\top}(\vecfont{\beta}_1^* - \vecfont{\beta}_2))^2 } > \frac{1}{2}
\end{equation*}
and 
\begin{equation*}
\sigma_{\min}(\Sigma) \geq (1 - \frac{2}{\pi}).
\end{equation*}
Using Hoeffding's inequality, with probability greater than $1 - e^{-\frac{1}{8}p_1|\mathcal{S}_t|} $, we have the bound $|J_1\cap J_1^*| \geq \frac{1}{4}p_1|\mathcal{S}_t|$. By a standard concentration argument (see, e.g., \cite{Vershynin} Corollary 50), we conclude that for any $\eta \in (0, 1 - \frac{2}{\pi})$, there exists a constant $c_3$, such that if 
\begin{equation}
\label{Ncond1}
|\mathcal{S}_t| \geq c_3\frac{k}{\eta p_1},
\end{equation}
then
\begin{equation}
\label{boundA}
A \leq \frac{4}{(1-\frac{2}{\pi}-\eta)p_1|\mathcal{S}_t|},
\end{equation}
with probability at least $1 - e^{-k}$. \\

\textbf{Bounding $B$.}
Let $Q := X^{\top}(W - WW^*)X$. We have 
\begin{equation*}
B^2 \leq \sigma_{\max}(Q) (\vecfont{\beta}_1^* - \vecfont{\beta}_2^*)^{\top}Q(\vecfont{\beta}_1^* - \vecfont{\beta}_2^*).
\end{equation*}
Moreover,
\begin{align*}
& \qquad (\vecfont{\beta}_1^* - \vecfont{\beta}_2^*)^{\top}Q(\vecfont{\beta}_1^* - \vecfont{\beta}_2^*) \\
& =  \sum_{i \in J_1\cap  J_2^*}(\vecfont{x}_i^{\top}(\vecfont{\beta}_1^* - \vecfont{\beta}_2^*))^2 \\
& \leq  \sum_{i \in J_1\bigcap J_2^*} 2(\vecfont{x}_i^{\top}(\vecfont{\beta}_1^* - \vecfont{\beta}_1))^2 + 2(\vecfont{x}_i^{\top}(\vecfont{\beta}_2^* - \vecfont{\beta}_1))^2 \\
& \leq  \sum_{i \in J_1\bigcap J_2^*} 2(\vecfont{x}_i^{\top}(\vecfont{\beta}_1^* - \vecfont{\beta}_1))^2 + 2(\vecfont{x}_i^{\top}(\vecfont{\beta}_2^* - \vecfont{\beta}_2))^2.
\end{align*}
The last inequality results from the decision rule labeling $\beta_1$ and $\beta_2$. This immediately implies that
\begin{equation}
\label{boundB}
B \leq 2\sigma_{\max}(Q){\rm err}.
\end{equation}
Using Lemma \ref{cd}, $\sigma_{\max}(\expec{ \vecfont{x}_i\vecfont{x}_i^{\top}| i \in J_1 \cap J_2^*}) \leq 2$. Following Theorem 39 in \cite{Vershynin}, we claim that there exist constants $c_4, c_5$ such that with probability greater than $1 - 2e^{-c_4k}$,
\begin{equation*}
\sigma_{\max}(Q) \leq |J_1\cap J_2^*| (2 + \max(\hat{\eta},{\hat{\eta}}^2))
\end{equation*}
where $\hat{\eta} = c_5\sqrt{\frac{k}{|J_1\cap J_2^*|}}$.
Letting $c_6 = 2 + c_5^2$, we have 
\begin{equation*}
\sigma_{\max}(Q) \leq c_6 \max(k, |J_1\cap J_2^*|).
\end{equation*}
Now using again Lemma \ref{cd}, we find
\begin{equation*}
\expec{|J_1\cap J_2^*|} \leq \frac{2{\rm err}^{(t-1)}}{\|\vecfont{\beta}_1^* - \vecfont{\beta}_2^*\|_2} p_2|\mathcal{S}_t|.
\end{equation*}
By Hoeffding's inequality, with high probability
\begin{equation*}
|J_1\cap J_2^*| \leq 2\expec{|J_1\cap J_2^*|}.
\end{equation*}

Now we can combine the bounds on $A$ (\ref{boundA}) and on $B$ (\ref{boundB}). Setting $\eta = (1-\frac{2}{\pi})/2$, when
\begin{equation}
\label{errcond2}
{\rm err} \leq \frac{0.18}{64c_6}p_1\|\vecfont{\beta}_1^*-\vecfont{\beta}_2^*\|_2,
\end{equation}
and
\begin{equation}
\label{Ncond2}
|\mathcal{S}_t| \geq \frac{16c_6}{0.18}\frac{k}{p_1},
\end{equation}
we conclude that
\begin{equation*}
\|\vecfont{\beta}_1^+ - \vecfont{\beta}_1^*\|_2 \leq \frac{1}{2}{\rm err}.
\end{equation*}
Repeating the steps for $\beta_2^+$, we obtain a similar result, and hence we conclude: ${\rm err}^+ \leq \frac{1}{2} {\rm err}$, as claimed.

%\input{sparsification}

\clearpage
\newpage
\bibliography{reference1}%\balancecolumns
\bibliographystyle{abbrv}
\clearpage
\newpage
\appendix
\section{Appendix}

We provide several technical results used in the main portion of the paper. For ease of reading, we reproduce the statements of the results as well as providing their proofs.

\subsection{Proof of Proposition \ref{prop:hardness}}

{\bf Proposition \ref{prop:hardness}}.  {\em Even in the noiseless setting, the general mixed regression problem is NP-hard. Specifically, deciding if a noiseless mixed regression problem specified by $(\mathbf{y},X)$ has a solution, $\beta_1, \beta_2$, is NP-hard.}

\begin{proof} The proof follows via a reduction from the so-called {\sc SubsetSum} problem, which is known to be NP-hard \cite{GareyJohnson}.
Recall that the {\sc SubsetSum} decision problem is as follows: given $k$ numbers, $a_1,\dots,a_k$ in $\mathbb{R}$, decide if there exists a partition $S \subseteq [k]$ such that 
$$
\sum_{i \in S} a_i = \sum_{j \in S^c} a_j.
$$
We show that if we can solve the mixed linear equations problem in polynomial time, then we can solve the {\sc SubsetSum} problem, which would thus imply that $P = NP$.

Given $\mathbf{a} = (\begin{array}{ccc} a_1 & \dots & a_k \end{array} )^{\top}$, we must design a matrix $X$, and output variable $\mathbf{y}$, such that if we could solve the mixed linear equation problem specified by $(\mathbf{y},X)$, then we could decide the subset sum problem on $\{a_1,\dots,a_k\}$. To this end, we define:
$$
X = \left[ \begin{array}{c} I_k \\ I_k \\ \begin{array}{ccc} 1 & \cdots & 1 \end{array} \end{array} \right], \qquad \mathbf{y} = \left( \begin{array}{c} \mathbf{a} \\ \mathbf{0}_{k \times 1} \\ \sum_i a_i / 2 \end{array} \right).
$$
Here, $I_k$ denotes the $k \times k$ identity matrix, $\mathbf{1}_{k \times 1}$ the $k \times 1$ vector of $1$'s, and similarly, $\mathbf{0}_{k \times 1}$ the $k \times 1$ vector of $0$'s. Finding a solution to the mixed linear equations problem amounts to finding a subset $S \subseteq [2k+1]$ of the $2k+1$ constraints, and vectors $\beta^{(1)}, \beta^{(2)} \in \mathbb{R}^k$, so that $\beta^{(1)}$ satisfies the equalities $X_S \beta^{(1)} = \mathbf{y}_S$, and $\beta^{(2)}$ the equalities $X_{S^c}\beta_2 = \mathbf{y}_{S^c}$. Note that $S$ cannot contain $i$ and $k+i$, since these equalities are mutually exclusive. The consequence is that we have $\beta_i^{(1)} \in \{0,1\}$, with $\beta_i^{(1)} = 1-\beta_i^{(2)}$. Thus if the first $2k$ constraints are satisfied, the final constraint, therefore, can only be satisfied if we have
$$
\sum_{i \in S} a_i = \sum_i a_i \beta_i^{(1)} = \sum_j a_j \beta_j^{(2)} = \sum_{j \in S^c} a_j,
$$
thus proving the result.
\end{proof}	 

\subsection{Proof of Proposition \ref{prop:exactReco}}
It's equivalent to show that $J_b = J_b^*, b = 1,2$. Let's consider $b = 1$, that is for all $p_1*|\mathcal{S}_t|$ samples that are generated by $y = \vecfont{x}^T\vecfont{\beta}_1^*$. For simplicity, let $\vecfont{\beta}_1, \vecfont{\beta}_2$ denote $\vecfont{\beta}_1^{(t-1)}, \vecfont{\beta}_2^{(t-1)}$, we need
\[
\big(\vecfont{x}^T(\vecfont{\beta}_1^* - \vecfont{\beta}_1)\big)^2 < \big(\vecfont{x}^T(\vecfont{\beta}_1^* - \vecfont{\beta}_2)\big)^2.
\]
From lemma \ref{cd},
\begin{align}
\prob{
\big(\vecfont{x}^T(\vecfont{\beta}_1^* - \vecfont{\beta}_1)\big)^2 < \big(\vecfont{x}^T(\vecfont{\beta}_1^* - \vecfont{\beta}_2)\big)^2
} & \geq 1 - \frac{\|\vecfont{\beta}_1^* - \vecfont{\beta}_1\|_2}{\|\vecfont{\beta}_1^* - \vecfont{\beta}_2\|_2} \\
& \geq 1 - 2\frac{\|\vecfont{\beta}_1^* - \vecfont{\beta}_1\|_2}{\|\vecfont{\beta}_1^* - \vecfont{\beta}_2^*\|_2} \\
& \geq 1 - \frac{2c_1}{k^2}.
\end{align}
Then we use union bound for $p_1*|\mathcal{S}_t|$ samples in $J_1^*$,
\[
\prob{
\big(\vecfont{x}_i^T(\vecfont{\beta}_1^* - \vecfont{\beta}_1)\big)^2 < \big(\vecfont{x}_i^T(\vecfont{\beta}_1^* - \vecfont{\beta}_2)\big)^2, \;for\;all\; i \in J_1^*
} \geq 1 - p_1c_2k \times \frac{2c_1}{k^2} \geq 1 - \frac{c'}{k}.
\] 
So all samples are correctly clustered with high probability. 

As $\frac{1}{\min(p_1,p_2)}k < |\mathcal{S}_t|$, number of samples in $J_1$ and $J_2$ are both greater than $k$. Therefore, least square solution reveals the ground truth. In other words, $err^{(t)} = 0$.

\subsection{Proof of Lemma \ref{cd}}

(1) \\
Without loss of generality, we assume $T\{u,v\} = T\{\vecfont{e}_1,\vecfont{e}_2\}$. Let $x_1,x_2$ denote $\vecfont{x}^T\vecfont{e}_1,\vecfont{x}^T\vecfont{e}_2$. As $x_1,x_2$ are independent Gaussian random variables, we have $x_1 = A \cos\theta, x_2 = A \sin\theta$, where $A$ is Rayleigh random variable, and $\theta$ is uniformly distributed over $[0,2\pi)$. Conditioning on $(\mathbf{x}^Tu)^2 > (\mathbf{x}^Tv)^2$, the range of $\theta$ is truncated to be $[\theta_0,\theta_0+\alpha_{(u,v)}]\cup [\theta_0+\pi,\theta_0+\pi+\alpha_{(u,v)}]$ for some $\theta_0$. It is not hard to see the eigenvalues of covariance matrix of $(x_1,x_2)$ are $1 + \frac{\sin\alpha_{(u,v)}}{\alpha_{(u,v)}}, 1 - \frac{\sin\alpha_{(u,v)}}{\alpha_{(u,v)}}$. As the rest if the eigenvalues of $\Sigma$ are 1, this completes the proof.\\
(2) \\
Note that
\begin{equation*}
\prob{(\mathbf{x}^Tu)^2 > (\mathbf{x}^Tv)^2} = \frac{\alpha_{(u,v)}}{\pi}.
\end{equation*} 
If $\|u\|_2 > \|v\|_2$, $\alpha_{(u,v)} > \frac{\pi}{2}$, when $\|u\|_2 < \|v\|_2$, 
\begin{equation*}
\cos{\alpha_{(u,v)}} \geq \frac{\|v\|_2^2 - \|u\|_2^2}{\|u\|_2^2+\|v\|_2^2}.
\end{equation*}
Note that for any $\alpha \in [0,\pi/2]$, $\alpha \leq \frac{\pi}{2} \sin\alpha$. We have
\begin{equation*}
\prob{(\mathbf{x}^Tu)^2 > (\mathbf{x}^Tv)^2} \leq \frac{1}{2}\sin{\alpha_{(u,v)}} \leq \frac{\|u\|_2\|v\|_2}{\|u\|_2^2 + \|v\|_2^2} \leq \frac{\|u\|_2}{\|v\|_2}.
\end{equation*}

\subsection{Supporting Lemmas}

\begin{lemma} \label{lem:concentration}
For any given $\delta > 0$, let $G$ denote the grid points, at resolution $\delta$, of the unit circle on the subspace spanned by the top two eigenvectors of $M$, formed with $|S_*|$ samples. Then, there exists an absolute constant $c$ such that if
\begin{equation*}
|S_*| \geq c(1/\tilde{\delta})^2k\log^2k,
\end{equation*}
where 
\begin{equation*}
\tilde{\delta} = \frac{\delta^2}{384}(1-\sqrt{1-4(1-\iprod{\vecfont{\beta}_1^*}{\vecfont{\beta}_2^*}^2)p_1p_2}),
\end{equation*}
then
\begin{equation*}
\min_{\vecfont{a}\in G}\|\vecfont{\beta}_i^* - \vecfont{a}\| \leq \delta, i=1,2,
\end{equation*}
with probability at least $1-\order{\frac{1}{k^2}}$.
\end{lemma}
\begin{proof}
In order to prove the result, we make use of standard concentration results.
%\begin{lemma}
%\label{heavy tail}
%Let $A$ be an $N \times k$ matrix whose rows $A_i$ are independent vectors with covariance %matrix $\Sigma$ and supported in some centered Euclidean ball whose radius we denote by $\sqrt{m}$. Let $\delta \in (0,1)$ and $t \geq 1$. Then the following holds with probability at least $1 - k^{-t^2}$:
%\begin{equation*}
%{\rm if}\; N \geq c(t/\delta)^2\|\Sigma\|^{-1}m \log k, \,\, {\rm then}\;\|\Sigma_N - %\Sigma\| \leq \delta\|\Sigma\|,
%\end{equation*}
%where $c > 0 $ is an absolute constant. 
%\end{lemma}

Let $\Sigma = \expec{M}$. We observe that $\prob{|y|>\sqrt{2\alpha\log{k}}} \leq n^{-\alpha}$, $\prob{\|\vecfont{x}\|_2^2\geq 3k} \leq e^{-k/3}$. Suppose $N$ is much less than $\order{k^{10}}$, where the constant is  arbitrarily chosen here. Set $\alpha = 12$. Then with probability at least $1 - \order{\frac{1}{k^2}}$, The vectors $y_i\vecfont{x}_i$ are all supported in a ball with radius $\sqrt{72k\log{k}}$.  Directly following theorem 5.44 in \cite{Vershynin}, we claim that when $N > C(1/\tilde{\delta})^2k\log^2k$,
\begin{equation*}
\|M - \Sigma\| \leq \tilde{\delta}\|\Sigma\| \leq 3\tilde{\delta}. 
\end{equation*}
We use $\sigma_i(A)$ to denote the $i$'th biggest eigenvalue of the positive semidefinite matrix $A$. By simple algebraic calculation we get $\sigma_1(\Sigma) = 2 + \kappa$, $\sigma_2(\Sigma) = 2 - \kappa$, where $\kappa = \sqrt{1-4(1-\iprod{\beta_1^*}{\beta_2^*}^2)p_1p_2}$. The top two eigenvectors of $\Sigma$ are denoted as $\vecfont{v}_1^*$, $\vecfont{v}_2^*$. We use $\vecfont{v}_1$, $\vecfont{v}_2$ to denote the top two eigenvectors of $M$. Lemma \ref{MatBound} yields that 
\begin{align*}
\|\vecfont{v}_i^* - \mathcal{P}_{T(\vecfont{v}_1,\vecfont{v}_2)}\vecfont{v}_i^*\|_2^2 & \leq \frac{12\tilde{\delta}}{\sigma_2(M) - \sigma_3(M)} \\
& \leq \frac{12\tilde{\delta}}{\sigma_2(\Sigma) - \sigma_3(\Sigma) - 6\tilde{\delta}} \\
& = \frac{12\tilde{\delta}}{1 - \kappa - 6\tilde{\delta}} \\
& = \frac{24\tilde{\delta}}{1-\kappa}, i=1,2.
\end{align*}
The last inequality holds when $\tilde{\delta} \leq \frac{1-\kappa}{12}$. 
Using the fact that for any two vectors $\vecfont{a},\vecfont{b}$,$\|\vecfont{a}+\vecfont{b}\|_2^2 \leq 2\|\vecfont{a}\|_2^2 + 2\|\vecfont{b}\|_2^2$, we conclude that
\begin{equation*}
\|\vecfont{\beta}_i^* - \mathcal{P}_{T(\vecfont{v}_1,\vecfont{v}_2)}\vecfont{\beta}_i^*\|_2^2  \leq \frac{48\tilde{\delta}}{1-\kappa}, i = 1, 2.
\end{equation*}
Let $w = \|\vecfont{\beta}_i^* - \mathcal{P}_{T(\vecfont{u},\vecfont{v})}\vecfont{\beta}_i^*\|_2$. Then, by simple geometric relation, 
\begin{align*}
\min_{\vecfont{a} \in \mathbb{S}^{k-1}\cap T_{(\vecfont{u},\vecfont{v})}}\|\vecfont{a} - \vecfont{\beta}_i^*\|_2^2 &\leq 2 - 2\sqrt{1-w^2} \\
& \leq 2w^2 \\
& \leq (\frac{\epsilon}{2})^2, i = 1,2.
\end{align*}
Consider the $\delta$-resolution grid $G$. We observe that for any point in $\mathbb{S}^{k-1}\cap T_{(\vecfont{u},\vecfont{v})}$, there exists a point in $G$ that is within $\delta/2$ away from it. By triangle inequality, we end up with
\begin{equation}
\min_{\vecfont{a}\in W}\|\vecfont{a} - \vecfont{\beta}_i^*\|_2 \leq \delta.
\end{equation}
\end{proof}

\begin{lemma} \label{sim}
Let $\hat{\vecfont{\beta}}_1,\hat{\vecfont{\beta}}_2$ be any two given vectors with error defined by ${\rm err} := \max_{i=1,2} \|\hat{\beta}_i - \beta_i^*\| $. There exist constants $c_1, c_2 > 0$ such that as long as we have enough testing samples, 
\begin{equation*}
|\mathcal{S}_+| \geq c_1 k/\min\{p_1,p_2\},
\end{equation*}
then with probability at least $1 - \order{e^{-c_2k}}$
\begin{equation*}
\sqrt{\frac{\mathcal{L}_+(\hat{\beta}_1,\hat{\beta}_2)}{|\mathcal{S}_+|}} \leq 1.1\, {\rm err}
\end{equation*}
and
\begin{equation*}
\sqrt{\frac{\mathcal{L}_+(\hat{\beta}_1,\hat{\beta}_2)}{|\mathcal{S}_+|}} \geq \frac{1}{5}\sqrt{\min\{p_1,p_2\}}\min\left \{{\rm err},\frac{1}{2}\|\vecfont{\beta}_1^* - \vecfont{\beta}_2^*\|_2\right \}.
\end{equation*}
\end{lemma}
\begin{proof}
Our notation here, namely, $J_1,J_2,J_1^*,J_2^*$, is consistent with proof of Theorem \ref{thm:em}. Note that we have: 
\begin{equation*}
\mathcal{L}(\vecfont{\beta}_1, \vecfont{\beta}_2) = \sum_i\min_{z_i}z_i(y_i - \vecfont{x}_i^T\vecfont{\beta}_1)^2 + (1-z_i)(y_i - \vecfont{x}_i^T\vecfont{\beta}_2)^2.
\end{equation*}
For the upper bound, we assign label $z_i$ as the true label. Then,
\begin{equation*}
\mathcal{L} \leq \sum_{i \in J_1^*} (\vecfont{x}_i^T(\vecfont{\beta}_1^* - \vecfont{\beta}_1))^2 + \sum_{i \in J_2^*} (\vecfont{x}_i^T(\vecfont{\beta}_2^* - \vecfont{\beta}_2))^2.
\end{equation*}
When $|\mathcal{S}_+| \geq C\frac{k}{\min\{p_1,p_2\}}$, then the number of samples in set $J_1^*$,$J_2^*$ is also greater than $Ck$. Following standard concentration results, there exist constants $C, c_1$, such that with probability greater than $1 - e^{-c_1k}$, we have 
\begin{equation*}
\|\frac{1}{p_j|\mathcal{S}_+|}\sum_{i \in J_j^*}(\vecfont{x}_i\vecfont{x}_i^T) - I \| \leq 0.21, j = 1,2.
\end{equation*}
We have 
\begin{align*}
\mathcal{L} & \leq 1.21p_1|\mathcal{S}_+|\|\vecfont{\beta}_1 - \vecfont{\beta}_1^*\|_2^2 + 1.21p_2|\mathcal{S}_+|\|\vecfont{\beta}_2 - \vecfont{\beta}_2^*\|_2^2 \\
& \leq 1.21|\mathcal{S}_+| {\rm err}^2.
\end{align*}
For the lower bound, we observe that 
\begin{equation*}
\mathcal{L} =  \underbrace {\sum_{i \in J_1\cap J_1^*} (\vecfont{x}_i^T(\vecfont{\beta}_1 - \vecfont{\beta}_1^*))^2 + \sum_{i \in J_2\cap J_1^*} (\vecfont{x}_i^T(\vecfont{\beta}_2 - \vecfont{\beta}_1^*))^2}_{A1} + \underbrace{\sum_{i \in J_1\cap J_2^*} (\vecfont{x}_i^T(\vecfont{\beta}_1 - \vecfont{\beta}_2^*))^2 + \sum_{i \in J_2\cap J_2^*} (\vecfont{x}_i^T(\vecfont{\beta}_2 - \vecfont{\beta}_2^*))^2}_{A2}.
\end{equation*}
 
First we consider the first term, A1. Note a simple fact that $\|\vecfont{\beta}_1 - \vecfont{\beta}_1^*\|_2 < \|\vecfont{\beta}_2 - \vecfont{\beta}_1^*\|_2$ or $\|\vecfont{\beta}_1 - \vecfont{\beta}_1^*\|_2 > \|\vecfont{\beta}_2 - \vecfont{\beta}_1^*\|_2$.
In the first case, from Lemma \ref{cd}, $\expec{|J_1\cap J_1^*|} \geq \frac{1}{2}p_1|\mathcal{S}_+|$. From Hoeffding's inequality and concentration result (see proof of Lemma \ref{cd} for similar techniques), for any $\delta \in (0,1-\frac{2}{\pi})$, there exist constants $C',c_1'$, such that when $N \geq C'k/p_1$, with probability at least $1 - e^{-c_1'k}$,
\begin{equation*}
\sum_{i \in J_1\cap J_1^*} (\vecfont{x}_i^T(\vecfont{\beta}_1 - \vecfont{\beta}_1^*))^2 \geq \frac{1}{4}p_1|\mathcal{S}_+|(1-\frac{1}{\pi}-\delta)\|\vecfont{\beta}_1 - \vecfont{\beta}_1^*\|_2^2.
\end{equation*}
In the second case, we have a similar result:
\begin{equation*}
\sum_{i \in J_2\cap J_1^*} (\vecfont{x}_i^T(\vecfont{\beta}_2 - \vecfont{\beta}_1^*))^2 \geq \frac{1}{4}p_1|\mathcal{S}_+|(1-\frac{1}{\pi}-\delta)\|\vecfont{\beta}_2 - \vecfont{\beta}_1^*\|_2^2.
\end{equation*}
Let $1 - \frac{2}{\pi} - \delta = 0.3$ and choose $C',c_1'$ to let the above results also hold for A2. We then conclude that when $N > C'\frac{k}{\min\{p_1,p_2\}}$,
\begin{equation}
\label{roughlb}
\mathcal{L} \geq \frac{0.3}{4}p_1|\mathcal{S}_+|\min\{\|\beta_1 - \beta_1^*\|_2^2, \|\beta_2 - \beta_1^*\|_2^2\} + \frac{0.3}{4}p_2|\mathcal{S}_+|\min\{\|\beta_1 - \beta_2^*\|_2^2, \|\beta_2 - \beta_2^*\|_2^2\}.
\end{equation}
When $\|\beta_1 - \beta_1^*\|_2 <  \|\beta_2 - \beta_1^*\|_2$ and $\|\beta_2 - \beta_2^*\|_2 <  \|\beta_1 - \beta_2^*\|_2$, (\ref{roughlb}) implies 
\begin{equation} \label{lb1}
\mathcal{L} \geq \frac{1}{25}\min\{p_1,p_2\}|\mathcal{S}_+| {\rm err}^2.
\end{equation}
When $\|\beta_1 - \beta_1^*\|_2 >  \|\beta_2 - \beta_1^*\|_2$ and $\|\beta_2 - \beta_2^*\|_2 <  \|\beta_1 - \beta_2^*\|_2$, we have
\begin{align} 
\mathcal{L} & \geq \frac{1}{25}\min\{p_1,p_2\}|\mathcal{S}_+|(\|\beta_2 - \beta_1^*\|_2^2 + \|\beta_2 - \beta_2^*\|_2^2) \\
& \geq \frac{1}{25}\min\{p_1,p_2\}|\mathcal{S}_+|\frac{1}{4}\|\beta_1^* - \beta_2^*\|_2^2. \label{lb2}
\end{align}
Note that it is impossible for $\|\beta_1 - \beta_1^*\|_2 >  \|\beta_2 - \beta_1^*\|_2$ and $\|\beta_2 - \beta_2^*\|_2 >  \|\beta_1 - \beta_2^*\|_2$ both to be true. Otherwise, we could switch the subscripts of the two $\beta$'s. Putting (\ref{lb1}) and (\ref{lb2}) together, we complete the proof.
\end{proof}

\begin{lemma}
\label{MatBound}
Suppose symmetric matrix $\Sigma \in \mathbb{R}^{n\times n}$ has eigenvalues $\lambda_1 \geq \lambda_2 > \lambda_3...$ with corresponding normalized eigenvectors denoted as $u_1, u_2, u_3,...$. Let $M$ be another symmetric matrix with eigenvalues: $\tilde{\lambda}_1 \geq \tilde{\lambda}_2 > \tilde{\lambda}_3...$ and eigenvectors $\tilde{u}_1, \tilde{u}_2, \tilde{u}_3,...$. (a) Let $span\{u_1,u_2\}$ denote the hyperplane spanned by $u_1$ $u_2$. If $\|M - \Sigma\|_2 \leq \varepsilon$, for $\varepsilon < \frac{\lambda_2 - \lambda_3}{2}$ we have
\begin{equation}
\label{spaceAppr}
\|\tilde{u}_i - \mathcal{P}_{T(u_1,u_2)}\tilde{u}_i\|_2^2 \leq \frac{4\varepsilon}{\lambda_2 - \lambda_3}, i = 1,2.
\end{equation} 

Moreover, if $\lambda_1 \ne \lambda_2$,
\begin{equation}
\label{u1Appr}
\|u_1 -\tilde{u}_1 \|_2^2 \leq \frac{4\epsilon}{\lambda_1 - \lambda_2}
\end{equation}

\begin{equation}
\label{u2Appr}
\|u_2 -\tilde{u}_2 \|_2^2 \leq \frac{4\epsilon}{\lambda_1 - \lambda_2} +  \frac{8\epsilon}{\lambda_2 - \lambda_3}
\end{equation}

\end{lemma}
\begin{proof} Suppose $\tilde{u}_1 = \alpha_1u_1 + \beta_1u_2+ \gamma_1w$, $\tilde{u}_2 = \alpha_2u_1 + \beta_2u_2+ \gamma_2v$, where $w,v$ are vector orthogonal to $span\{u_1,u_2\}$. We have $\alpha_1^2 + \beta_1^2 + \gamma_1^2 = \alpha_2^2 + \beta_2^2 + \gamma_2^2 = 1$. Since $\|M - \Sigma\|_2 \leq \varepsilon$, 
\begin{equation}
\label{e1}
\tilde{u}_1^TM\tilde{u}_1 \geq \lambda_1 - \varepsilon
\end{equation} 
\begin{eqnarray}
\label{e2}
\tilde{u}_1^TM\tilde{u}_1 & \leq & \tilde{u}_1^T(M-\Sigma)\tilde{u}_1 + \tilde{u}_1^T\Sigma \tilde{u}_1 \\
 & \leq & \varepsilon + \tilde{u}_1^T\Sigma \tilde{u}_1.
\end{eqnarray} 
Combining (\ref{e1}) and (\ref{e2}), using $\tilde{u}_1^T\Sigma \tilde{u}_1 = \alpha_1^2\lambda_1 + \beta_1^2\lambda_2 + \gamma_1^2\lambda_3$, we get
\begin{equation}
\label{e3}
\alpha_1^2\lambda_1 + \beta_1^2\lambda_2 + \gamma_1^2\lambda_3 \geq \lambda_1 - 2\varepsilon
\end{equation}
Since $\alpha_1^2\lambda_1 + \beta_1^2\lambda_2 + \gamma_1^2\lambda_3 \leq (1 - \gamma_1^2)\lambda_1 + \gamma_1^2\lambda_3$, it implies that
\begin{equation}
\label{e4}
\gamma_1^2 \leq \frac{2\varepsilon}{\lambda_1 - \lambda_3} \leq \frac{2\varepsilon}{\lambda_2 - \lambda_3}.
\end{equation}
We assume $\lambda_1 \neq \lambda_2$. Otherwise , the above inequality also holds for $\tilde{u}_2$, then the proof of (\ref{spaceAppr}) is completed. By using another upper bound $\alpha_1^2\lambda_1 + \beta_1^2\lambda_2 + \gamma_1^2\lambda_3 \leq \alpha_1^2\lambda_1 + (1 - \alpha_1^2)\lambda_2$, the following inequality $\alpha_1^2$ holds
\begin{equation}
\label{alpha_1}
\alpha_1^2 \geq 1 - \frac{2\varepsilon}{\lambda_1 - \lambda_2}.
\end{equation}
Note $\|\tilde{u}_2 - \mathcal{P}_{T(u_1,u_2)}\tilde{u}_2\|_2^2 = \gamma_1^2$, we get the distance bound of $u_1$. Next, we show the bound for $\tilde{u}_2$.
Similar to (\ref{e3}), 
\begin{equation}
\label{e5}
\alpha_2^2\lambda_1 + \beta_2^2\lambda_2 + \gamma_2^2\lambda_3 \geq \lambda_2 - 2\varepsilon.
\end{equation}
Again, by using $\alpha_2^2\lambda_1 + \beta_2^2\lambda_2 + \gamma_2^2\lambda_3 \leq \alpha_2^2\lambda_1 + (1 - \alpha_2^2)\lambda_2$, we get
\begin{equation}
\label{e6}
\gamma_2^2 \leq \frac{2\varepsilon+\alpha_2^2(\lambda_1-\lambda_2)}{\lambda_2 - \lambda_3}.
\end{equation}
We use the condition that $\tilde{u}_1$\vspace{1 pt} $\tilde{u}_2$ are orthogonal. Hence, $\alpha_1^2\alpha_2^2 \leq (1-\alpha_1^2)(1-\alpha_2^2)$. It is easy to see $\alpha_1^2 + \alpha_2^2 \leq 1$. Plugging it into (\ref{e6}) and using (\ref{alpha_1}) result in
\begin{equation}
\label{e7}
\gamma_2^2 \leq \frac{4\varepsilon}{\lambda_2 - \lambda_3}.
\end{equation}
Through (\ref{e4}) and (\ref{e7}), we complete the proof of (\ref{spaceAppr}).

Using some intermediate results, we derive the bounds for eigenvectors in the case $\lambda_1 \ne \lambda_2$.
\begin{align*}
\|u_1 -\tilde{u}_1 \|_2^2 & = (1 - \alpha_1)^2 + \beta_1 ^2 + \gamma_1^2 \\
& = (1 - \alpha_1)^2 + 1 - \alpha_1^2 \\
& \leq 2(1 - \alpha_1^2) \\
& \leq \frac{4\epsilon}{\lambda_1 - \lambda_2}.
\end{align*}
The last inequality follows from (\ref{alpha_1}).

Similarly,
\begin{align*}
\|u_2 -\tilde{u}_2 \|_2^2 & \leq 2(1 - \beta_2^2) \\
& = 2(\alpha_2^2 + \gamma_2^2) \\
& \leq 2(1 - \alpha_1^2 + \gamma_2^2) \\
& \leq \frac{4\epsilon}{\lambda_1 - \lambda_2} +  \frac{8\epsilon}{\lambda_2 - \lambda_3}.
\end{align*}
We obtain the last inequality from (\ref{alpha_1}) and (\ref{e7}).
\end{proof}

\end{document}